\documentclass[accepted]{uai2023}
\usepackage{natbib}
\usepackage{mathtools}
\usepackage{booktabs}
\usepackage{tikz}

\usepackage{algorithm}
\usepackage{algorithmicx}
\usepackage[noend]{algpseudocode}
\usepackage{amsfonts}
\usepackage{amsmath}
\usepackage{amssymb}
\usepackage{amsthm}
\usepackage{bbm}
\usepackage{bm}
\usepackage{color}
\usepackage{dirtytalk}
\usepackage{dsfont}
\usepackage{enumerate}
\usepackage{graphicx}
\usepackage{listings}
\usepackage{mathtools}
\usepackage{subfigure}
\usepackage{times}
\usepackage{url}
\usepackage{xspace}

\usepackage[bookmarks=false]{hyperref}
\hypersetup{
  pdffitwindow=true,
  pdfstartview={FitH},
  pdfnewwindow=true,
  colorlinks,
  linktocpage=true,
  linkcolor=blue,
  urlcolor=blue,
  citecolor=blue
}
\usepackage[capitalize,noabbrev]{cleveref}

\usepackage[textsize=tiny]{todonotes}

\usepackage{thmtools}
\declaretheorem[name=Theorem,refname={Theorem,Theorems},Refname={Theorem,Theorems}]{theorem}
\declaretheorem[name=Lemma,refname={Lemma,Lemmas},Refname={Lemma,Lemmas},sibling=theorem]{lemma}

\newcommand{\cA}{\mathcal{A}}

\newcommand{\cN}{\mathcal{N}}

\newcommand{\E}[1]{\mathbb{E} \left[#1\right]}
\newcommand{\condE}[2]{\mathbb{E} \left[#1 \,\middle|\, #2\right]}

\newcommand{\prob}[1]{\mathbb{P} \left(#1\right)}
\newcommand{\condprob}[2]{\mathbb{P} \left(#1 \,\middle|\, #2\right)}

\newcommand{\var}[1]{\mathrm{var} \left[#1\right]}

\newcommand{\abs}[1]{\left|#1\right|}
\newcommand{\ceils}[1]{\left\lceil#1\right\rceil}

\newcommand{\floors}[1]{\left\lfloor#1\right\rfloor}
\newcommand{\I}[1]{\mathds{1} \! \left\{#1\right\}}

\newcommand{\set}[1]{\left\{#1\right\}}

\DeclareMathOperator*{\argmax}{arg\,max\,}

\mathchardef\mhyphen="2D

\newcommand{\agapev}{\ensuremath{\tt A\mhyphen GapE\mhyphen V}\xspace}
\newcommand{\alg}{\ensuremath{\tt Alg}\xspace}
\newcommand{\gape}{\ensuremath{\tt GapE}\xspace}
\newcommand{\gapev}{\ensuremath{\tt GapE\mhyphen V}\xspace}
\newcommand{\sh}{\ensuremath{\tt SH}\xspace}
\newcommand{\shadavar}{\ensuremath{\tt SHAdaVar}\xspace}
\newcommand{\shvar}{\ensuremath{\tt SHVar}\xspace}
\newcommand{\unif}{\ensuremath{\tt Unif}\xspace}
\newcommand{\vbr}{\ensuremath{\tt VBR}\xspace}

\title{Fixed-Budget Best-Arm Identification with Heterogeneous Reward Variances}

\author[ ]{{Anusha Lalitha}}
\author[ ]{Kousha Kalantari}
\author[ ]{Yifei Ma}
\author[ ]{Anoop Deoras}
\author[ ]{Branislav Kveton}
\affil[ ]{AWS AI Labs}
\affil[ ]{\texttt{\{anlalith,kkalant,yifeim,adeoras,bkveton\}@amazon.com}}

\begin{document}

\maketitle

\begin{abstract}
We study the problem of best-arm identification (BAI) in the fixed-budget setting with heterogeneous reward variances. We propose two variance-adaptive BAI algorithms for this setting: \shvar for known reward variances and \shadavar for unknown reward variances. The key idea in our algorithms is to adaptively allocate more budget to arms with higher reward variances. The main algorithmic novelty is in the design of \shadavar, which allocates budget greedily based on overestimating unknown reward variances. We bound the probabilities of misidentifying best arms in both \shvar and \shadavar. Our analyses rely on novel lower bounds on the number of arm pulls in BAI that do not require closed-form solutions to the budget allocation problem. One of our budget allocation problems is equivalent to the optimal experiment design with unknown variances and thus of a broad interest. We also evaluate our algorithms on synthetic and real-world problems. In most settings, \shvar and \shadavar outperform all prior algorithms.
\end{abstract}

\section{Introduction}
\label{sec:introduction}

The problem of \emph{best-arm identification (BAI)} in the \emph{fixed-budget} setting is a \emph{pure exploration} bandit problem which can be briefly described as follows. An agent interacts with a stochastic multi-armed bandit with $K$ arms and its goal is to identify the arm with the highest mean reward within a fixed budget $n$ of arm pulls \citep{bubeck09pure,audibert10best}. This problem arises naturally in many applications in practice, such as online advertising, recommender systems, and vaccine tests \citep{lattimore19bandit}. It is also common in applications where observations are costly, such as Bayesian optimization \citep{krause08nearoptimal}. Another commonly studied setting is \emph{fixed-confidence} BAI \citep{evendar06action,soare14bestarm}. Here the goal is to identify the best arm within a prescribed confidence level while minimizing the budget. Some works also studied both settings \citep{gabillon12best,karnin13almost,kaufmann16complexity}.

Our work can be motivated by the following example. Consider an A/B test where the goal is to identify a movie with the highest average user rating from a set of $K$ movies. This problem can be formulated as BAI by treating the movies as arms and user ratings as stochastic rewards. Some movies get either unanimously good or bad ratings, and thus their ratings have a low variance. Others get a wide range of ratings, because they are rated highly by their target audience and poorly by others; and hence their ratings have a high variance. For this setting, we can design better BAI policies that take the variance into account. Specifically, movies with low-variance ratings can be exposed to fewer users in the A/B test than movies with high-variance ratings.

An analogous synthetic example is presented in \cref{fig:synthetic}. In this example, reward variances increase with mean arm rewards for a half of the arms, while the remaining arms have very low variances. The knowledge of the reward variances can be obviously used to reduce the number of pulls of arms with low-variance rewards. However, in practice, the reward variances are rarely known in advance, such as in our motivating A/B testing example, and this makes the design and analysis of variance-adaptive BAI algorithms challenging. We revisit these two examples in our empirical studies in \cref{sec:experiments}.

We propose and analyze two variance-adaptive BAI algorithms: \shvar and \shadavar. \shvar assumes that the reward variances are known and is a stepping stone for our fully-adaptive BAI algorithm \shadavar, which estimates them. \shadavar utilizes high-probability upper confidence bounds on the reward variances. Both algorithms are motivated by sequential halving (\sh) of \citet{karnin13almost}, a near-optimal solution for fixed-budget BAI with homogeneous reward variances.

Our main contributions are:
\begin{itemize}
  \item We design two variance-adaptive algorithms for fixed-budget BAI: \shvar for known reward variances and \shadavar for unknown reward variances. \shadavar is only a third algorithm for this setting \citep{gabillon11multibandit,faella20rapidly} and only a second that can be implemented as analyzed \citep{faella20rapidly}. The key idea in \shadavar is to solve a budget allocation problem with unknown reward variances by a greedy algorithm that overestimates them. This idea can be applied to other elimination algorithms in the cumulative regret setting \citep{auer10ucb} and is of independent interest to the field of optimal experiment design \citep{pukelsheim93optimal}.
  \item We prove upper bounds on the probability of misidentifying the best arm for both \shvar and \shadavar. The analysis of \shvar extends that of \citet{karnin13almost} to heterogeneous variances. The analysis of \shadavar relies on a novel lower bound on the number of pulls of an arm that scales linearly with its unknown reward variance. This permits an analysis of sequential halving without requiring a closed form for the number of pulls of each arm.
  \item We evaluate our methods empirically on Gaussian bandits and the MovieLens dataset \citep{movielens}. In most settings, \shvar and \shadavar outperform all prior algorithms.
\end{itemize}

The paper is organized as follows. In \cref{sec:setting}, we present the fixed-budget BAI problem. We present our algorithms in \cref{sec:algorithms} and analyze them in \cref{sec:analysis}. The algorithms are empirically evaluated in \cref{sec:experiments}. We review prior works in \cref{sec:related work} and conclude in \cref{sec:conclusions}.

\section{Setting}
\label{sec:setting}

We use the following notation. Random variables are capitalized, except for Greek letters like $\mu$. For any positive integer $n$, we define $[n] = \set{1, \dots, n}$. The indicator function is denoted by $\I{\cdot}$. The $i$-th entry of vector $v$ is $v_i$. If the vector is already indexed, such as $v_j$, we write $v_{j, i}$. The big O notation up to logarithmic factors is $\tilde{O}$. 

We have a stochastic bandit with $K$ arms and denote the set of arms by $\cA = [K]$. When the arm is pulled, its reward is drawn i.i.d.\ from its reward distribution. The reward distribution of arm $i \in \cA$ is sub-Gaussian with mean $\mu_i$ and variance proxy $\sigma^2_i$. The \emph{best arm} is the arm with the highest mean reward,
\begin{align*}
  \textstyle
  i_*
  = \argmax_{i \in \cA} \mu_i\,.
\end{align*}
Without loss of generality, we make an assumption that the arms are ordered as $\mu_1 > \mu_2 \geq \ldots \geq \mu_K$. Therefore, arm $i_* = 1$ is a unique best arm. The agent has a budget of $n$ observations and the goal is to identify $i_*$ as accurately as possible after pulling all arms $n$ times. Specifically, let $\hat{I}$ denote the arm returned by the agent after $n$ pulls. Then our objective is to minimize the \emph{probability of misidentifying the best arm} $\prob{\hat{I} \neq i_*}$, which we also call a \emph{mistake probability}. This setting is known as \emph{fixed-budget BAI} \citep{bubeck09pure,audibert10best}. When observations are costly, it is natural to limit them by a fixed budget $n$.

Another commonly studied setting is \emph{fixed-confidence BAI} \citep{evendar06action,soare14bestarm}. Here the agent is given an upper bound on the mistake probability $\delta$ as an input and the goal is to attain $\prob{\hat{I} \neq i_*} \leq \delta$ at minimum budget $n$. Some works also studied both the fixed-budget and fixed-confidence settings \citep{gabillon12best,karnin13almost,kaufmann16complexity}.

\begin{figure}[t]
  \centering
  \includegraphics[width=0.48\textwidth]{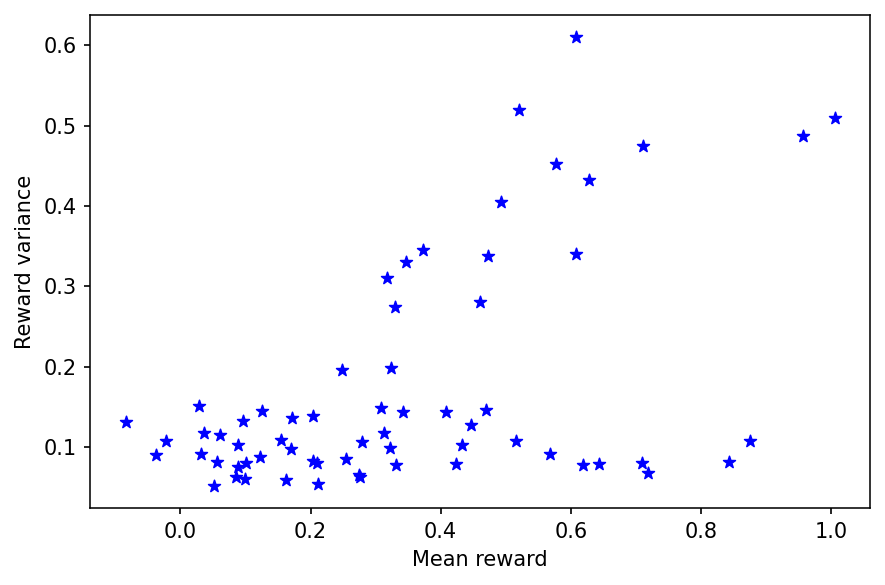}
  \caption{Mean rewards and variances for $K = 64$ arms in the Gaussian bandit in \cref{sec:synthetic experiments}.}
  \label{fig:synthetic}
\end{figure}

\section{Algorithms}
\label{sec:algorithms}

A near-optimal solution for fixed-budget BAI with homogeneous reward variances is sequential halving \citep{karnin13almost}. The key idea is to sequentially eliminate suboptimal arms in $\log_2 K$ stages. In each stage, all arms are pulled equally and the worst half of the arms are eliminated at the end of the stage. At the end of the last stage, only one arm $\hat{I}$ remains and that arm is the estimated best arm.

\begin{algorithm}[t]
  \caption{Meta-algorithm for sequential halving.}
  \label{alg:meta}
  \begin{algorithmic}[1]
    \State \textbf{Input:} Budget $n$, base algorithm \alg
    \Statex
    \State Number of stages $m \gets \ceils{\log_2 K}$
    \State $\cA_1 \gets \cA$
    \For{$s = 1, \dots, m$}
      \State Per-stage budget $\displaystyle n_s \gets \floors{n / m}$ 
      \For{$t = 1, \dots, n_s$}
        \State $I_{s, t} \gets \alg(s, t)$
        \State Observe reward $Y_{s, t, I_{s, t}}$ of arm $I_{s, t}$
      \EndFor
      \For{$i \in \cA_s$}
        \State $\displaystyle N_{s, i} \gets \sum_{t = 1}^{n_s} \I{I_{s, t} = i}$
        \State $\displaystyle \hat{\mu}_{s, i} \gets \frac{1}{N_{s, i}} \sum_{t = 1}^{n_s} \I{I_{s, t} = i} Y_{s, t, i}$
      \EndFor
      \State $\cA_{s + 1} \gets \set{\textrm{$\ceils{\abs{\cA_s} / 2}$ arms $i \in \cA_s$ with highest $\hat{\mu}_{s, i}$}}$
    \EndFor
    \Statex
    \State \textbf{Output:} The last remaining arm $\hat{I}$ in $\cA_{m + 1}$
  \end{algorithmic}
\end{algorithm}

\begin{algorithm}[t]
  \caption{\sh: Pulled arm in sequential halving.}
  \label{alg:sh}
  \begin{algorithmic}[1]
    \State \textbf{Input:} Stage $s$, round $t$
    \Statex
    \State $k \gets (t - 1) \bmod \abs{\cA_s} + 1$
    \State $I_{s, t} \gets \textrm{$k$-th arm in $\cA_s$}$
    \Statex
    \State \textbf{Output:} Arm to pull $I_{s, t}$
  \end{algorithmic}
\end{algorithm}

\begin{algorithm}[t]
  \caption{\shvar: Pulled arm in sequential halving with known heterogeneous reward variances.}
  \label{alg:shvar}
  \begin{algorithmic}[1]
    \State \textbf{Input:} Stage $s$, round $t$
    \Statex
    \For{$i \in \cA_s$}
      \State $\displaystyle N_{s, t, i} \gets \sum_{\ell = 1}^{t - 1} \I{I_{s, \ell} = i}$
    \EndFor
    \State $\displaystyle I_{s, t} \gets \argmax_{i \in \cA_s} \frac{\sigma_i^2}{N_{s, t, i}}$
    \Statex
    \State \textbf{Output:} Arm to pull $I_{s, t}$
  \end{algorithmic}
\end{algorithm}

The main algorithmic contribution of our work is that we generalize sequential halving of \citet{karnin13almost} to heterogeneous reward variances. All of our algorithms can be viewed as instances of a meta-algorithm (\cref{alg:meta}), which we describe in detail next. Its inputs are a \emph{budget} $n$ on the number of observations and base algorithm $\alg$. The meta-algorithm has $m$ stages (line 2) and the budget is divided equally across the stages, with a \emph{per-stage budget} $n_s = \floors{n / m}$ (line 5). In stage $s$, all \emph{remaining arms} $\cA_s$ are pulled according to $\alg$ (lines 6--8). At the end of stage $s$, the worst half of the remaining arms, as measured by their estimated mean rewards, is eliminated (lines 9--12). Here $Y_{s, t, i}$ is the \emph{stochastic reward} of arm $i$ in round $t$ of stage $s$, $I_{s, t} \in \cA_s$ is the \emph{pulled arm} in round $t$ of stage $s$, $N_{s, i}$ is the \emph{number of pulls} of arm $i$ in stage $s$, and $\hat{\mu}_{s, i}$ is its \emph{mean reward estimate} from all observations in stage $s$.

The sequential halving of \citet{karnin13almost} is an instance of \cref{alg:meta} for $\alg = \sh$. The pseudocode of \sh, which pulls all arms in stage $s$ equally, is in \cref{alg:sh}. We call the resulting algorithm \sh. This algorithm misidentifies the best arm with probability \citep{karnin13almost}
\begin{align}
  \prob{\hat{I} \neq 1}
  \leq 3 \log_2 K \exp\left[- \frac{n}{8 H_2 \log_2 K}\right]\,,
  \label{eq:sh error bound}
\end{align}
where
\begin{align}
  H_2
  = \max_{i \in \cA \setminus \set{1}} \frac{i}{\Delta_i^2}
  \label{eq:h2}
\end{align}
is a \emph{complexity parameter} and $\Delta_i = \mu_1 - \mu_i$ is the \emph{suboptimality gap} of arm $i$. The bound in \eqref{eq:sh error bound} decreases as budget $n$ increases and problem complexity $H_2$ decreases.

\sh is near optimal only in the setting of homogeneous reward variances. In this work, we study the general setting where the reward variances of arms vary, potentially as extremely as in our motivating example in \cref{fig:synthetic}. In this example, \sh would face arms with both low and high variances in each stage. A variance-adaptive \sh could adapt its budget allocation in each stage to the reward variances and thus eliminate suboptimal arms more effectively.

\subsection{Known Heterogeneous Reward Variances}
\label{sec:shvar}

We start with the setting of known reward variances. Let
\begin{align}
  \sigma_i^2
  = \var{Y_{s, t, i}}
  = \E{(Y_{s, t, i} - \mu_i)^2}
  \label{eq:reward variance}
\end{align}
be a known reward variance of arm $i$. Our proposed algorithm is an instance of \cref{alg:meta} for $\alg = \shvar$. The pseudocode of \shvar is in \cref{alg:shvar}. The key idea is to pull the arm with the highest variance of its mean reward estimate. The variance of the mean reward estimate of arm $i$ in round $t$ of stage $s$ is $\sigma_i^2 / N_{s, t, i}$, where $\sigma_i^2$ is the reward variance of arm $i$ and $N_{s, t, i}$ is the number of pulls of arm $i$ up to round $t$ of stage $s$. We call the resulting algorithm \shvar.

Note that \sh is an instance of \shvar. Specifically, when all $\sigma_i = \sigma$ for some $\sigma > 0$, \shvar pulls all arms equally, as in \sh. \shvar can be also viewed as pulling any arm $i$ in stage $s$ for
\begin{align}
  N_{s, i}
  \approx \frac{\sigma_i^2}{\sum_{j \in \cA_s} \sigma_j^2} n_s
  \label{eq:shvar allocation}
\end{align}
times. This is stated formally and proved below.

\begin{lemma}
\label{lem:shvar allocation} Fix stage $s$ and let the ideal number of pulls of arm $i \in \cA_s$ be
\begin{align*}
  \lambda_{s, i}
  = \frac{\sigma_i^2}{\sum_{j \in \cA_s} \sigma_j^2} n_s\,.
\end{align*}
Let all $\lambda_{s, i}$ be integers. Then \shvar pulls arm $i$ in stage $s$ exactly $\lambda_{s, i}$ times.
\end{lemma}
\begin{proof}
First, suppose that \shvar pulls each arm $i$ exactly $\lambda_{s, i}$ times. Then the variances of all mean reward estimates at the end of stage $s$ are identical, because
\begin{align*}
  \frac{\sigma_i^2}{N_{s, i}}
  = \frac{\sigma_i^2}{\lambda_{s, i}}
  = \frac{\sigma_i^2}{\frac{\sigma_i^2}{\sum_{j \in \cA_s} \sigma_j^2} n_s}
  = \frac{\sum_{j \in \cA_s} \sigma_j^2}{n_s}\,.
\end{align*}
Now suppose that this is not true. This implies that there exists an over-pulled arm $i \in \cA_s$ and an under-pulled arm $k \in \cA_s$ such that
\begin{align}
  \frac{\sigma_i^2}{N_{s, i}}
  < \frac{\sum_{j \in \cA_s} \sigma_j^2}{n_s}
  < \frac{\sigma_k^2}{N_{s, k}}\,.
  \label{eq:non-uniform estimator variance}
\end{align}
Since arm $i \in \cA_s$ is over-pulled and $\lambda_{s, i}$ is an integer, there must exist a round $t \in [n_s]$ such that
\begin{align*}
  \frac{\sigma_i^2}{N_{s, t, i}}
  = \frac{\sigma_i^2}{\lambda_{s, i}}
  = \frac{\sum_{j \in \cA_s} \sigma_j^2}{n_s}\,.
\end{align*}
Let $t$ be the last round where this equality holds, meaning that arm $i$ is pulled in round $t$.

Now we combine the second inequality in \eqref{eq:non-uniform estimator variance} with $N_{s, k} \geq N_{s, t, k}$, which holds by definition, and get
\begin{align*}
  \frac{\sum_{j \in \cA_s} \sigma_j^2}{n_s}
  < \frac{\sigma_k^2}{N_{s, k}}
  \leq \frac{\sigma_k^2}{N_{s, t, k}}\,.
\end{align*}
The last two sets of inequalities lead to a contradiction. On one hand, we know that arm $i$ is pulled in round $t$. On the other hand, we have $\sigma_i^2 / N_{s, t, i} < \sigma_k^2 / N_{s, t, k}$, which means that arm $i$ cannot be pulled. This completes the proof.
\end{proof}

\cref{lem:shvar allocation} says that each arm $i \in \cA_s$ is pulled $O(\sigma_i^2)$ times. Since the mean reward estimate of arm $i$ at the end of stage $s$ has variance $\sigma_i^2 / N_{s, i}$, the variances of all estimates at the end of stage $s$ are identical, $\left(\sum_{i \in \cA_s} \sigma_i^2\right) / n_s$. This relates our problem to the G-optimal design \citep{pukelsheim93optimal}. Specifically, the $G$-optimal design for independent experiments $i \in \cA_s$ is an allocation of observations $(N_{s, i})_{i \in \cA_s}$ such that $\sum_{i \in \cA_s} N_{s, i} = n_s$ and the maximum variance
\begin{align}
  \max_{i \in \cA_s} \frac{\sigma_i^2}{N_{s, i}}
  \label{eq:maximum variance}
\end{align}
is minimized. This happens precisely when all $\sigma_i^2 / N_{s, i}$ are identical, when $N_{s, i} = \lambda_{s, i}$ for $\lambda_{s, i}$ in \cref{lem:shvar allocation}.

\subsection{Unknown Heterogeneous Reward Variances}
\label{sec:shadavar}

Our second proposal is an algorithm for unknown reward variances. One natural idea, which is expected to be practical but hard to analyze, is to replace $\sigma_i^2$ in \shvar with its empirical estimate from the past $t - 1$ rounds in stage $s$,
\begin{align*}
  \hat{\sigma}_{s, t, i}^2
  = \frac{1}{N_{s, t, i} - 1} \sum_{\ell = 1}^{t - 1} \I{I_{s, \ell} = i}
  (Y_{s, \ell, i} - \hat{\mu}_{s, t, i})^2\,,
\end{align*} 
where
\begin{align*}
  \hat{\mu}_{s, t, i}
  = \frac{1}{N_{s, t, i}} \sum_{\ell = 1}^{t - 1} \I{I_{s, \ell} = i} Y_{s, \ell, i}
\end{align*}
is the empirical mean reward of arm $i$ in round $t$ of stage $s$. This design would be hard to analyze because $\hat{\sigma}_{s, t, i}$ can underestimate $\sigma_i$, and thus is not an optimistic estimate.

The key idea in our solution is to act optimistically using an \emph{upper confidence bound (UCB)} on the reward variance. To derive it, we make an assumption that the reward noise is Gaussian. Specifically, the reward of arm $i$ in round $t$ of stage $s$ is distributed as $Y_{s, t, i} \sim \cN(\mu_i, \sigma_i^2)$. This allows us to derive the following upper and lower bounds on the unkown variance $\sigma_i^2$.

\begin{algorithm}[t]
  \caption{\shadavar: Pulled arm in sequential halving with unknown heterogeneous reward variances.}
  \label{alg:shadavar}
  \begin{algorithmic}[1]
    \State \textbf{Input:} Stage $s$, round $t$
    \Statex
    \If{$t \leq \abs{\cA_s} (4 \log(1 / \delta) + 1)$}
      \State $k \gets (t - 1) \bmod \abs{\cA_s} + 1$
      \State $I_{s, t} \gets \textrm{$k$-th arm in $\cA_s$}$
    \Else
      \For{$i \in \cA_s$}
        \State $\displaystyle N_{s, t, i} \gets \sum_{\ell = 1}^{t - 1} \I{I_{s, \ell} = i}$
      \EndFor
      \State $\displaystyle I_{s, t} \gets \argmax_{i \in \cA_s} \frac{U_{s, t, i}}{N_{s, t, i}}$
    \EndIf
    \Statex
    \State \textbf{Output:} Arm to pull $I_{s, t}$
  \end{algorithmic}
\end{algorithm}

\begin{lemma}
\label{lem:concentration} Fix stage $s$, round $t \in [n_s]$, arm $i \in \cA_s$, and failure probability $\delta \in (0, 1)$. Let
\begin{align*}
  N
  = N_{s, t, i} - 1
\end{align*}
and suppose that $N > 4 \log(1 / \delta)$. Then
\begin{align*}
  \prob{\sigma_i^2
  \geq \frac{\hat{\sigma}_{s, t, i}^2}{1 - 2 \sqrt{\frac{\log(1 / \delta)}{N}}}}
  \leq \delta
\end{align*}
holds with probability at least $1 - \delta$. Analogously,
\begin{align*}
  \prob{\hat{\sigma}_{s, t, i}^2
  \geq \sigma_i^2 \left[1 + 2 \sqrt{\frac{\log(1 / \delta)}{N}} +
  \frac{2 \log(1 / \delta)}{N}\right]}
  \leq \delta
\end{align*}
holds with probability at least $1 - \delta$.
\end{lemma}
\begin{proof}
The first claim is proved as follows. By Cochran's theorem, we have that $\hat{\sigma}_{s, t, i}^2 N / \sigma_i^2$ is a $\chi^2$ random variable with $N$ degrees of freedom. Its concentration was analyzed in \citet{laurent00adaptive}. More specifically, by (4.4) in \citet{laurent00adaptive}, an immediate corollary of their Lemma 1, we have
\begin{align*}
  \prob{N - \frac{\hat{\sigma}_{s, t, i}^2 N}{\sigma_i^2}
  \geq 2 \sqrt{N \log(1 / \delta)}}
  \leq \delta\,.
\end{align*}
Now we divide both sides in the probability by $N$, multiply by $\sigma_i^2$, and rearrange the formula as
\begin{align*}
  \prob{\sigma_i^2 \left(1 - 2 \sqrt{\log(1 / \delta) / N}\right)
  \geq \hat{\sigma}_{s, t, i}^2}
  \leq \delta\,.
\end{align*}
When $1 - 2 \sqrt{\log(1 / \delta) / N} > 0$, we can divide both sides by it and get the first claim in \cref{lem:concentration}.

The second claim is proved analogously. Specifically, by (4.3) in \citet{laurent00adaptive}, an immediate corollary of their Lemma 1, we have
\begin{align*}
  \prob{\frac{\hat{\sigma}_{s, t, i}^2 N}{\sigma_i^2} - N
  \geq 2 \sqrt{N \log(1 / \delta)} + 2 \log(1 / \delta)}
  \leq \delta\,.
\end{align*}
Now we divide both sides in the probability by $N$, multiply by $\sigma_i^2$, and obtain the second claim in \cref{lem:concentration}. This concludes the proof.
\end{proof}

By \cref{lem:concentration}, when $N_{s, t, i} > 4 \log(1 / \delta) + 1$,
\begin{align}
  U_{s, t, i}
  = \frac{\hat{\sigma}_{s, t, i}^2}
  {1 - 2 \sqrt{\frac{\log(1 / \delta)}{N_{s, t, i} - 1}}}
  \label{eq:variance ucb}
\end{align}
is a high-probability upper bound on the reward variance of arm $i$ in round $t$ of stage $s$, which holds with probability at least $1 - \delta$. This bound decreases as the number of observations $N_{s, t, i}$ increases and confidence $\delta$ decreases. To apply the bound across multiple stages, rounds, and arms, we use a union bound.

The bound in \eqref{eq:variance ucb} leads to our algorithm that overestimates the variance. The algorithm is an instance of \cref{alg:meta} for $\alg = \shadavar$. The pseudocode of \shadavar is in \cref{alg:shadavar}. To guarantee $N_{s, t, i} > 4 \log(1 / \delta) + 1$, we pull all arms $\cA_s$ in any stage $s$ for $4 \log(1 / \delta) + 1$ times initially. We call the resulting algorithm \shadavar.

Note that \shadavar can be viewed as a variant of \shvar where $U_{s, t, i}$ replaces $\sigma_i^2$. Therefore, it can also be viewed as solving the G-optimal design in \eqref{eq:maximum variance} without knowing reward variances $\sigma_i^2$; and \shadavar is of a broader interest to the optimal experiment design community \citep{pukelsheim93optimal}. We also note that the assumption of Gaussian noise in the design of \shadavar is limiting. To address this issue, we experiment with non-Gaussian noise in \cref{sec:movielens experiments}.

\section{Analysis}
\label{sec:analysis}

This section comprises three analyses. In \cref{sec:error bound shvar}, we bound the probability that \shvar, an algorithm that knows reward variances, misidentifies the best arm. In \cref{sec:error bound shvar2}, we provide an alternative analysis that does not rely on the closed form in \eqref{eq:shvar allocation}. Finally, in \cref{sec:error bound shadavar}, we bound the probability that \shadavar, an algorithm that learns reward variances, misidentifies the best arm.

All analyses are under the assumption of Gaussian reward noise. Specifically, the reward of arm $i$ in round $t$ of stage $s$ is distributed as $Y_{s, t, i} \sim \cN(\mu_i, \sigma_i^2)$.

\subsection{Error Bound of \shvar}
\label{sec:error bound shvar}

We start with analyzing \shvar, which is a stepping stone for analyzing \shadavar. To simplify the proof, we assume that both $m$ and $n_s$ are integers. We also assume that all budget allocations have integral solutions in \cref{lem:shvar allocation}.

\begin{theorem}
\label{thm:shvar} \shvar misidentifies the best arm with probability
\begin{align*}
  \prob{\hat{I} \neq 1}
  \leq 2 \log_2 K \exp\left[- \frac{n \Delta_{\min}^2}
  {4 \log_2 K \sum_{j \in \cA} \sigma_j^2}\right]\,,
\end{align*}
where $\Delta_{\min} = \mu_1 - \mu_2$ is the minimum gap.
\end{theorem}
\begin{proof}
The claim is proved in \cref{sec:shvar proof}. We follow the outline in \citet{karnin13almost}. The novelty is in extending the proof to heterogeneous reward variances. This requires a non-uniform budget allocation, where arms with higher reward variances are pulled more (\cref{lem:shvar allocation}).
\end{proof}

The bound in \cref{thm:shvar} depends on all quantities as expected. It decreases as budget $n$ and minimum gap $\Delta_{\min}$ increase, and the number of arms $K$ and variances $\sigma_j^2$ decrease. \shvar reduces to \sh in \citet{karnin13almost} when $\sigma_i^2 = 1 / 4$ for all arms $i \in \cA$. The bounds of \sh and \shvar become comparable when we apply $H_2 \leq K / \Delta_{\min}^2$ in \eqref{eq:sh error bound} and note that $\sum_{j \in \cA} \sigma_j^2 = K / 4$ in \cref{thm:shvar}. The extra factor of $8$ in the exponent of \eqref{eq:sh error bound} is due to a different proof, which yields a finer dependence on gaps.

\subsection{Alternative Error Bound of \shvar}
\label{sec:error bound shvar2}

Now we analyze \shvar differently. The resulting bound is weaker than that in \cref{thm:shvar} but its proof can be easily extended to \shadavar.

\begin{theorem}
\label{thm:shvar2} \shvar misidentifies the best arm with probability
\begin{align*}
  \prob{\hat{I} \neq 1}
  \leq 2 \log_2 K \exp\left[- \frac{(n - K \log K) \Delta_{\min}^2}
  {4 \sigma_{\max}^2 K \log_2 K}\right]\,,
\end{align*}
where $\Delta_{\min} = \mu_1 - \mu_2$ is the minimum gap and $\sigma_{\max}^2 = \max_{i \in \cA} \sigma_i^2$ is the maximum reward variance.
\end{theorem}
\begin{proof}
The claim is proved in \cref{sec:shvar2 proof}. The key idea in the proof is to derive a lower bound on the number of pulls of any arm $i$ in stage $s$, instead of using the closed form of $N_{s, i}$ in \eqref{eq:shvar allocation}. The lower bound is
\begin{align*}
  N_{s, i}
  \geq \frac{\sigma_i^2}{\sigma_{\max}^2} \left(\frac{n_s}{\abs{\cA_s}} - 1\right)\,.
\end{align*}
An important property of the bound is that it is $\Omega(\sigma_i^2 n_s)$, similarly to $N_{s, i}$ in \eqref{eq:shvar allocation}. Therefore, the rest of the proof is similar to that of \cref{thm:shvar}.
\end{proof}

As in \cref{thm:shvar}, the bound in \cref{thm:shvar2} depends on all quantities as expected. It decreases as budget $n$ and minimum gap $\Delta_{\min}$ increase, and the number of arms $K$ and maximum variance $\sigma_{\max}^2$ decrease. The bound approaches that in \cref{thm:shvar} when all reward variances are identical.

\subsection{Error Bound of \shadavar}
\label{sec:error bound shadavar}

Now we analyze \shadavar.

\begin{theorem}
\label{thm:shadavar} Suppose that $\delta < 1 / (K n)$ and
\begin{align*}
  n \geq
  K \log_2 K (4 \log(K n / \delta) + 1)\,.
\end{align*}
Then \shadavar misidentifies the best arm with probability
\begin{align*}
  \prob{\hat{I} \neq 1}
  \leq 2 \log_2 K \exp\left[- \alpha \frac{(n - K \log K) \Delta_{\min}^2}
  {4 \sigma_{\max}^2 K \log_2 K}\right]\,,
\end{align*}
where $\Delta_{\min}$ and $\sigma_{\max}^2$ are defined in \cref{thm:shvar2}, and
\begin{align*}
  \alpha
  = \frac{1 - 2 \sqrt{\frac{\log(K n / \delta)}{n / K - 2}}}
  {1 + 2 \sqrt{\frac{\log(K n / \delta)}{n / K - 2}} +
  \frac{2 \log(K n / \delta)}{n / K - 2}}\,.
\end{align*}
\end{theorem}
\begin{proof}
The claim is proved in \cref{sec:shadavar proof}. The key idea in the proof is to derive a lower bound on the number of pulls of any arm $i$ in stage $s$, similarly to that in \cref{thm:shvar2}. The lower bound is
\begin{align*}
  N_{s, i}
  \geq \frac{\sigma_i^2}{\sigma_{\max}^2} \alpha(\abs{\cA_s}, n_s, \delta)
  \left(\frac{n_s}{\abs{\cA_s}} - 1\right)
\end{align*}
and holds with probability at least $1 - \delta$. Since the bound is $\Omega(\sigma_i^2 n_s)$, as in the proof of \cref{thm:shvar2}, the rest of the proof is similar. The main difference from \cref{thm:shvar2} is in factor $\alpha(\abs{\cA_s}, n_s, \delta)$, which converges to $1$ as $n_s \to \infty$.
\end{proof}

The bound in \cref{thm:shadavar} depends on all quantities as expected. It decreases as budget $n$ and minimum gap $\Delta_{\min}$ increase, and the number of arms $K$ and maximum variance $\sigma_{\max}^2$ decrease. As $n \to \infty$, we get $\alpha \to 1$ and the bound converges to that in \cref{thm:shvar2}.

\section{Experiments}
\label{sec:experiments}

In this section, we empirically evaluate our proposed algorithms, \shvar and \shadavar, and compare them to algorithms from prior works. We choose the following baselines: uniform allocation (\unif), sequential halving (\sh) \citep{karnin13almost}, gap-based exploration (\gape) \citep{gabillon11multibandit}, gap-based exploration with variance (\gapev) \citep{gabillon11multibandit}, and variance-based rejects (\vbr) \citep{faella20rapidly}. 

\unif allocates equal budget to all arms and \sh was originally proposed for homogeneous reward variances. Neither \unif nor \sh can adapt to heterogenuous reward variances. \gape, \gapev and \vbr are variance-adaptive BAI methods from related works (\cref{sec:related work}). In \gape, we use $H$ from Theorem 1 of \citet{gabillon11multibandit}. In \gapev, we use $H$ from Theorem 2 of \citet{gabillon11multibandit}. Both \gape and \gapev assume bounded reward distributions with support $[0, b]$. We choose $b = \max_{i \in \cA} \mu_i + \sigma_i \sqrt{\log n}$, since this is a high-probability upper bound on the absolute value of $n$ independent observations from $\cN(\mu_i, \sigma_i^2)$. In \shadavar, we set $\delta = 0.05$, and thus our upper bounds on reward variances hold with probability $0.95$. In \vbr, $\gamma = 1.96$, which means that the mean arm rewards lie between their upper and lower bounds with probability $0.95$. \citet{faella20rapidly} showed that \vbr performs well with Gaussian noise when $\gamma \approx 2$. All reported results are averaged over $5\,000$ runs.

\gape and \gapev have $O(\exp[- c n / H])$ error bounds on the probability of misidentifying the best arm, where $n$ is the budget, $H$ is the complexity parameter, and $c = 1 / 144$ for \gape and $c = 1 / 512$ for \gapev. Our error bounds are $O(\exp[- c' n / H'])$, where $H'$ is a comparable complexity parameter and $c' = 1 / (4 \log_2 K)$. Even for moderate $K$, $c \ll c'$. Therefore, when \shvar and \shadavar are implemented as analyzed, they provide stronger guarantees on identifying the best arm than \gape and \gapev. To make the algorithms comparable, we set $H$ of \gape and \gapev to $H c / c'$, by increasing their confidence widths. Since $H$ is an input to both \gape and \gapev, note that they have an advantage over our algorithms that do not require it.

\subsection{Synthetic Experiments}
\label{sec:synthetic experiments}

Our first experiment is on a Gaussian bandit with $K$ arms. The mean reward of arm $i$ is $\mu_i = 1 - \sqrt{(i - 1) / K}$. We choose this setting because \sh is known to perform well in it. Specifically, note that the complexity parameter $H_2$ in \eqref{eq:h2} is minimized when $i / \Delta_i^2$ are equal for all $i \in \cA \setminus \set{1}$. For our $\mu_i$, $\Delta_i^2 = (i - 1) / K \approx i / K$ and thus $i / \Delta_i^2 \approx K$. We set the reward variance as $\sigma^2_i = 0.9 \mu^2_i + 0.1$ when arm $i$ is even and $\sigma^2_i = 0.1$ when arm $i$ is odd. We additionally perturb $\mu_i$ and $\sigma_i^2$ with additive $\cN(0, 0.05^2)$ and multiplicative $\mathrm{Unif}(0.5, 1.5)$ noise, respectively. We visualize the mean rewards $\mu_i$ and the corresponding variances $\sigma^2_i$, for $K = 64$ arms, in \cref{fig:synthetic}. The variances are chosen so that every stage of sequential halving involves both high-variance and low-variance arms. Therefore, an algorithm that adapts its budget allocation to the reward variances of the remaining arms eliminates the best arm with a lower probability than the algorithm that does not.

\begin{figure}[t]
  \centering
  \includegraphics[width=0.48\textwidth]{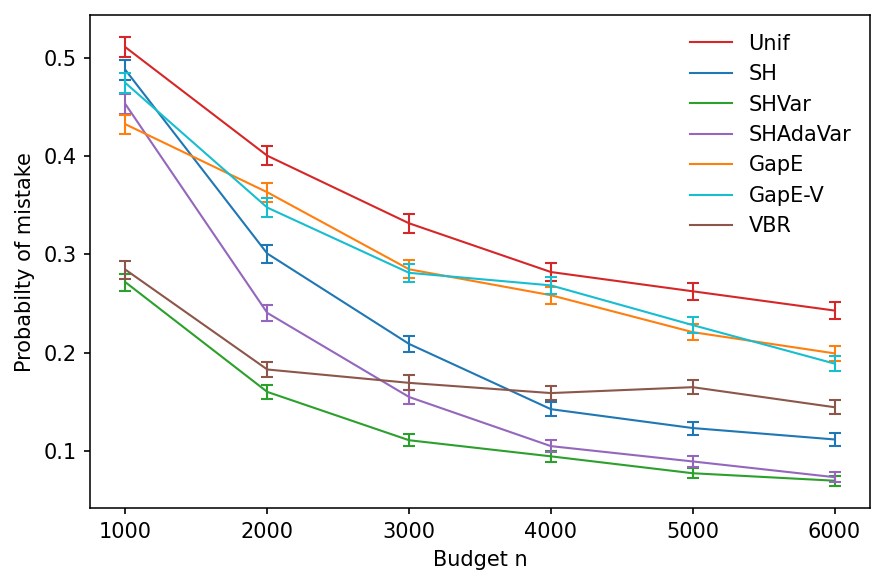}
  \caption{Probability of misidentifying the best arm in the Gaussian bandit in \cref{sec:synthetic experiments}, as budget $n$ increases. The number of arms is $K = 64$ and the results are averaged over $5\,000$ runs.}
  \label{fig:varying_n_K_64}
\end{figure}

In \cref{fig:varying_n_K_64}, we report the probability of misidentifying the best arm among $K = 64$ arms (\cref{fig:synthetic}) as budget $n$ increases. As expected, the naive algorithm \unif performs the worst. \gape and \gapev perform only slightly better. When the algorithms have comparable error guarantees to \shvar and \shadavar, their confidence intervals are too wide to be practical. \sh performs surprisingly well. As observed by \citet{karnin13almost} and confirmed by \citet{li18hyperband}, \sh is a superior algorithm in the fixed-budget setting because it aggressively eliminates a half of the remaining arms in each stage. Therefore, it outperforms \gape and \gapev. We note that \shvar outperforms all algorithms for all budgets $n$. For smaller budgets, \vbr outperforms \shadavar. However, as the budget $n$ increases, \shadavar outperforms \vbr; and without any additional information about the problem instance approaches the performance of \shvar, which knows the reward variances. This shows that our variance upper bounds improve quickly with larger budgets, as is expected based on the algebraic form in \eqref{eq:variance ucb}.

\begin{figure}[t]
  \centering
  \includegraphics[width=0.48\textwidth]{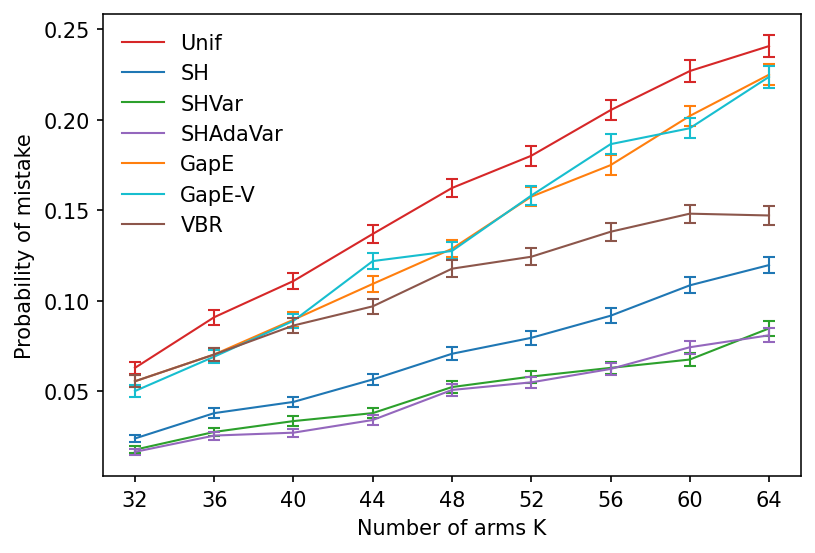}
  \caption{Probability of misidentifying the best arm in the Gaussian bandit in \cref{sec:synthetic experiments}, as the number of arms $K$ increases. The budget is fixed at $n = 5\,000$ and the results are averaged over $5\,000$ runs.}
  \label{fig:varying_K_n_4000_maxK_64}
\end{figure}

In the next experiment, we take same Gaussian bandit as in \cref{fig:varying_n_K_64}. The budget is fixed at $n = 5\,000$ and we vary the number of arms $K$ from $32$ to $64$. In \cref{fig:varying_K_n_4000_maxK_64}, we show the probability of misidentifying the best arm as the number of arms $K$ increases. We observe two major trends. First, the relative order of the algorithms, as measured by their probability of a mistake, is similar to \cref{fig:varying_n_K_64}. Second, all algorithms get worse as the number of arms $K$ increases because the problem instance becomes harder. This experiment shows that \shvar and \shadavar can perform well for a wide range of $K$, they have the lowest probabilities of a mistake for all $K$. While the other algorithms perform well at $K = 32$, their probability of a mistake is around $0.05$ or below; they perform poorly at $K = 64$, their probability of a mistake is above $0.1$.

\subsection{MovieLens Experiments}
\label{sec:movielens experiments}

Our next experiment is motivated by the A/B testing problem in \cref{sec:introduction}. The objective is to identify the movie with the highest mean rating from a pool of $K$ movies, where movies are arms and their ratings are rewards. The movies, users, and ratings are simulated using the MovieLens 1M dataset \citep{movielens}. This dataset contains one million ratings given by $6\,040$ users to $3\,952$ movies. We complete the missing ratings using low-rank matrix factorization with rank $5$, which is done using alternating least squares \citep{davenport16overview}. The result is a $6\,040 \times 3\,952$ matrix $M$, where $M_{i, j}$ is the estimated rating given by user $i$ to movie $j$.

\begin{figure}[t]
  \centering
  \includegraphics[width=0.48\textwidth]{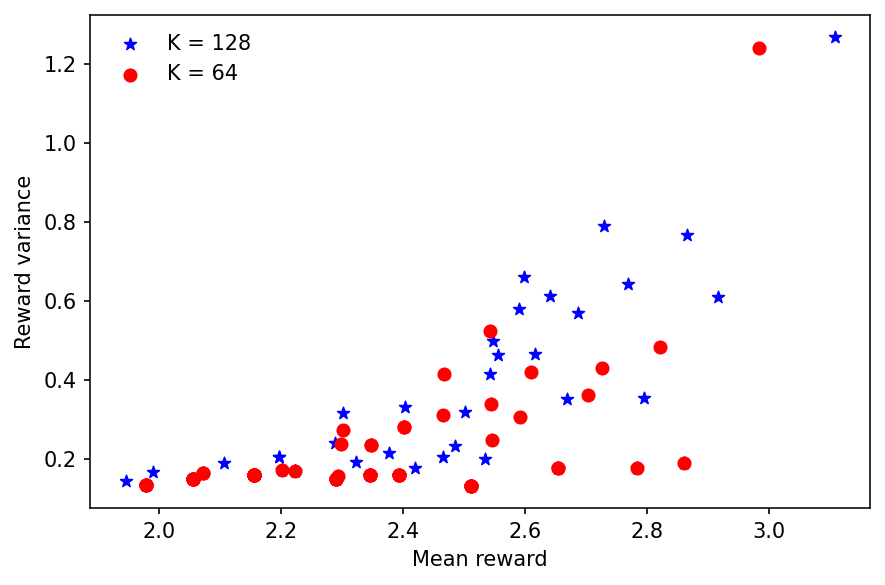}
  \caption{Means and variances of ratings of $K$ movies from the MovieLens dataset. A new sample is generated in each run of the experiment, as described in \cref{sec:movielens experiments}.}
  \label{fig:movielens}
\end{figure}

This experiment is averaged over $5\,000$ runs. In each run, we randomly choose new movies according to the following procedure. For all arms $i \in \cA$, we generate mean $\tilde{\mu}_i$ and variance $\tilde{\sigma}_i^2$ as described in \cref{sec:synthetic experiments}. Then, for each $i$, we find the closest movie in the MovieLens dataset with mean $\mu_i$ and variance $\sigma_i^2$, the movie that minimizes the distance $(\mu_i - \tilde{\mu}_i)^2 + (\sigma^2_i - \tilde{\sigma}_i^2)^2$. The means and variances of movie ratings from two runs are shown in \cref{fig:movielens}. As in \cref{sec:synthetic experiments}, the movies are selected so that sequential elimination with halving is expected to perform well. The variance of movie ratings in \cref{fig:movielens} is intrinsic to our domain: movies are often made for specific audiences and thus can have a huge variance in their ratings. For instance, a child may not like a horror movie, while a horror enthusiast would enjoy it. Because of this, an algorithm that adapts its budget allocation to the rating variances of the remaining movies can perform better. The last notable difference from \cref{sec:synthetic experiments} is that movie ratings are realistic. In particular, when arm $i$ is pulled, we choose a random user $j$ and return $M_{j, i}$ as its stochastic reward. Therefore, this experiment showcases the robustness of our algorithms beyond Gaussian noise.

\begin{figure}[t]
  \centering
  \includegraphics[width=0.48\textwidth]{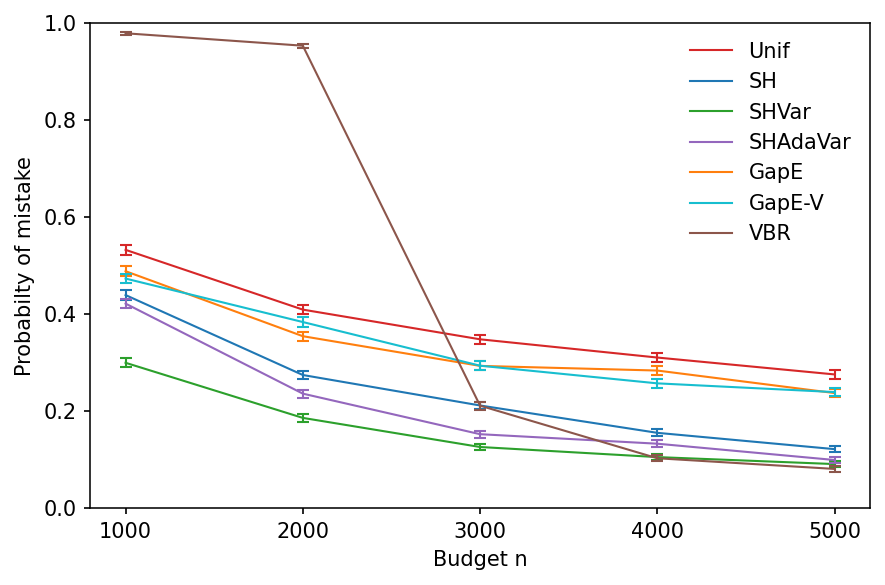}
  \caption{Probability of misidentifying the best movie in the MovieLens bandit in \cref{sec:movielens experiments}, as budget $n$ increases. The number of movies is $K = 64$ and the results are averaged over $5\,000$ runs.}
  \label{fig:varying_n_K_64_ML}
\end{figure}

In \cref{fig:varying_n_K_64_ML}, we report the probability of misidentifying the best movie from $K = 64$ as budget $n$ increases. \shvar and \shadavar perform the best for most budgets, although the reward distributions are not Gaussian. The relative performance of the algorithms is similar to \cref{sec:synthetic experiments}: \unif is the worst, and \gape and \gapev improve upon it. The only exception is \vbr: it performs poorly for smaller budgets, and on par with \shvar and \shadavar for larger budgets.

We increase the number of movies next. In \cref{fig:varying_n_K_128_ML}, we report the probability of misidentifying the best movie from $K = 128$ as budget $n$ increases. The trends are similar to $K = 64$, except that \vbr performs poorly for all budgets. This is because \vbr has $K$ stages and eliminates one arm per stage even when the number of observations is small. In comparison, our algorithms have $\log_2 K$ stages.

\begin{figure}[t]
  \centering
  \includegraphics[width=0.48\textwidth]{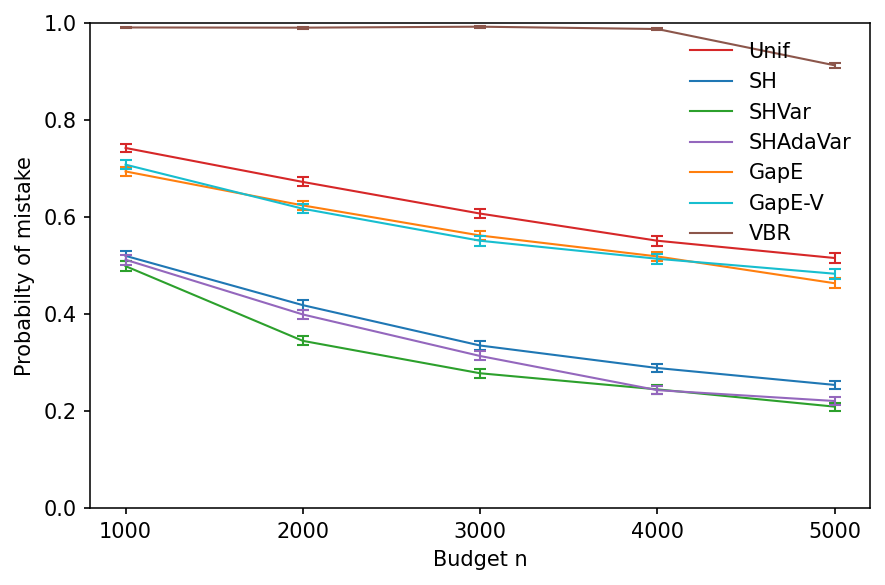}
  \caption{Probability of misidentifying the best movie in the MovieLens bandit in \cref{sec:movielens experiments}, as budget $n$ increases. The number of movies is $K = 128$ and the results are averaged over $5\,000$ runs.}
  \label{fig:varying_n_K_128_ML}
\end{figure}

\section{Related Work}
\label{sec:related work}

Best-arm identification has been studied extensively in both fixed-budget \citep{bubeck09pure,audibert10best} and fixed-confidence \citep{evendar06action} settings. The two closest prior works are \citet{gabillon11multibandit} and \citet{faella20rapidly}, both of which studied fixed-budget BAI with heterogeneous reward variances. All other works on BAI with heterogeneous reward variances are in the fixed-confidence setting \citep{lu21variancedependent,zhou22approximate,jourdan22dealing}.

The first work on variance-adaptive BAI was in the fixed-budget setting \citep{gabillon11multibandit}. This paper proposed algorithm \gapev and showed that its probability of mistake decreases exponentially with budget $n$. Our error bounds are comparable to \citet{gabillon11multibandit}. The main shortcoming of the analyses in \citet{gabillon11multibandit} is that they assume that the complexity parameter is known and used by \gapev. Since the complexity parameter depends on unknown gaps and reward variances, it is typically unknown in practice. To address this issue, \citet{gabillon11multibandit} introduced an adaptive variant of \gapev, \agapev, where the complexity parameter is estimated. This algorithm does not come with any guarantee.

The only other work that studied variance-adaptive fixed-budget BAI is \citet{faella20rapidly}. This paper proposed and analyzed a variant of successive rejects algorithm \citep{audibert10best}. Since \sh of \citet{karnin13almost} has a comparable error bound to successive rejects of \citet{audibert10best}, our variance-adaptive sequential halving algorithms have comparable error bounds to variance-adaptive successive rejects of \citet{faella20rapidly}. Roughly speaking, all bounds can be stated as $\exp[- n / H]$, where $H$ is a complexity parameter that depends on the number of arms $K$, their variances, and their gaps.

We propose variance-adaptive sequential halving for fixed-budget BAI. Our algorithms have state-of-the-art performance in our experiments (\cref{sec:experiments}). They are conceptually simpler than prior works \citep{gabillon11multibandit,faella20rapidly} and can be implemented as analyzed, unlike \citet{gabillon11multibandit}.

\section{Conclusions}
\label{sec:conclusions}

We study best-arm identification in the fixed-budget setting where the reward variances vary across the arms. We propose two variance-adaptive elimination algorithms for this problem: \shvar for known reward variances and \shadavar for unknown reward variances. Both algorithms proceed in stages and pull arms with higher reward variances more often than those with lower variances. While the design and analysis of \shvar are of interest, they are a stepping stone for \shadavar, which adapts to unknown reward variances. The novelty in \shadavar is in solving an optimal design problem with unknown observation variances. Its analysis relies on a novel lower bound on the number of arm pulls in BAI that does not require closed-form solutions to the budget allocation problem. Our numerical simulations show that \shvar and \shadavar are not only theoretically sound, but also competitive with state-of-the-art baselines.

Our work leaves open several questions of interest. First, the design of \shadavar is for Gaussian reward noise. The reason for this choice is that our initial experiments showed quick concentration and also robustness to noise misspecification. Concentration of general random variables with unknown variances can be analyzed using empirical Bernstein bounds \citep{maurer09empirical}. This approach was taken by \citet{gabillon11multibandit} and could also be applied in our setting. For now, to address the issue of Gaussian noise, we experiment with non-Gaussian noise in \cref{sec:movielens experiments}. Second, while our error bounds depend on all parameters of interest as expected, we do not provide a matching lower bound. When the reward variances are known, we believe that a lower bound can be proved by building on the work of \citet{carpentier16tight}. Finally, our algorithms are not contextual, which limits their application because many bandit problems are contextual \citep{li10contextual,wen15efficient,zong16cascading}.

\bibliography{References}

\clearpage
\onecolumn
\appendix

\section{Proof of Theorem \ref{thm:shvar}}
\label{sec:shvar proof}

First, we decompose the probability of choosing a suboptimal arm. For any $s \in [m]$, let $E_s = \set{1 \in \cA_{s + 1}}$ be the event that the best arm is not eliminated in stage $s$ and $\bar{E}_s$ be its complement. Then by the law of total probability,
\begin{align*}
  \prob{\hat{I} \neq 1}
  = \prob{\bar{E}_m}
  = \sum_{s = 1}^m \prob{\bar{E}_s, E_{s - 1} \dots, E_1}
  \leq \sum_{s = 1}^m \condprob{\bar{E}_s}{E_{s - 1} \dots, E_1}\,.
\end{align*}
We bound $\condprob{\bar{E}_s}{E_{s - 1} \dots, E_1}$ based on the observation that the best arm can be eliminated only if the estimated mean rewards of at least a half of the arms in $\cA_s$ are at least as high as that of the best arm. Specifically, let $\cA_s' = \cA_s \setminus \set{1}$ be the set of all arms in stage $s$ but the best arm and
\begin{align*}
  N_s'
  = \sum_{i \in \cA_s'} \I{\hat{\mu}_{s, i} \geq \hat{\mu}_{s, 1}}\,.
\end{align*}
Then by the Markov's inequality,
\begin{align*}
  \condprob{\bar{E}_s}{E_{s - 1} \dots, E_1}
  \leq \condprob{N_s' \geq \frac{n_s}{2}}{E_{s - 1} \dots, E_1}
  \leq \frac{2 \, \condE{N_s'}{E_{s - 1} \dots, E_1}}{n_s}\,.
\end{align*}
The key step in bounding the above expectation is understanding the probability that any arm has a higher estimated mean reward than the best one. We bound this probability next.

\begin{lemma}
\label{lem:arm error} For any stage $s \in [m]$ with the best arm, $1 \in \cA_s$, and any suboptimal arm $i \in \cA_s$, we have
\begin{align*}
  \prob{\hat{\mu}_{s, i} \geq \hat{\mu}_{s, 1}}
  \leq \exp\left[- \frac{n_s \Delta_i^2}{4 \sum_{j \in \cA_s} \sigma_j^2}\right]\,.
\end{align*}
\end{lemma}
\begin{proof}
The proof is based on concentration inequalities for sub-Gaussian random variables \citep{boucheron13concentration}. In particular, since $\hat{\mu}_{s, i} - \mu_i$ and $\hat{\mu}_{s, 1} - \mu_1$ are sub-Gaussian with variance proxies $\sigma_i^2 / N_{s, i}$ and $\sigma_1^2 / N_{s, 1}$, respectively; their difference is sub-Gaussian with a variance proxy $\sigma_i^2 / N_{s, i} + \sigma_1^2 / N_{s, 1}$. It follows that
\begin{align*}
  \prob{\hat{\mu}_{s, i} \geq \hat{\mu}_{s, 1}}
  & = \prob{\hat{\mu}_{s, i} - \hat{\mu}_{s, 1} \geq 0}
  = \prob{(\hat{\mu}_{s, i} - \mu_i) - (\hat{\mu}_{s, 1} - \mu_1) > \Delta_i} \\
  & \leq \exp\left[- \frac{\Delta_i^2}
  {2 \left(\frac{\sigma_i^2}{N_{s, i}} + \frac{\sigma_1^2}{N_{s, 1}}\right)}\right]
  = \exp\left[- \frac{n_s \Delta_i^2}{4 \sum_{j \in \cA_s} \sigma_j^2}\right]\,,
\end{align*}
where the last step follows from the definitions of $N_{s, i}$ and $N_{s, 1}$ in \cref{lem:shvar allocation}.
\end{proof}

The last major step is bounding $\condE{N_s'}{E_{s - 1} \dots, E_1}$ with the help of \cref{lem:arm error}. Starting with the union bound, we get
\begin{align*}
  \condE{N_s'}{E_{s - 1} \dots, E_1}
  & \leq \sum_{i \in \cA_s'} \prob{\hat{\mu}_{s, i} \geq \hat{\mu}_{s, 1}}
  \leq \sum_{i \in \cA_s'}
  \exp\left[- \frac{n_s \Delta_i^2}{4 \sum_{j \in \cA_s} \sigma_j^2}\right] \\
  & \leq n_s \max_{i \in \cA_s'}
  \exp\left[- \frac{n_s \Delta_i^2}{4 \sum_{j \in \cA_s} \sigma_j^2}\right]
  = n_s \exp\left[- \frac{n_s \min_{i \in \cA_s'} \Delta_i^2}
  {4 \sum_{j \in \cA_s} \sigma_j^2}\right]\,.
\end{align*}
Now we chain all inequalities and get
\begin{align*}
  \prob{\hat{I} \neq 1}
  \leq 2 \sum_{s = 1}^m \exp\left[- \frac{n_s \min_{i \in \cA_s'} \Delta_i^2}
  {4 \sum_{j \in \cA_s} \sigma_j^2}\right]\,.
\end{align*}
To get the final claim, we use that
\begin{align*}
  m
  = \log_2 K\,, \quad
  n_s
  = \frac{n}{\log_2 K}\,, \quad
  \min_{i \in \cA_s'} \Delta_i^2
  \geq \Delta_{\min}^2\,, \quad
  \sum_{j \in \cA_s} \sigma_j^2
  \leq \sum_{j \in \cA} \sigma_j^2\,.
\end{align*}
This concludes the proof.

\section{Proof of Theorem \ref{thm:shvar2}}
\label{sec:shvar2 proof}

This proof has the same steps as that in \cref{sec:shvar proof}. The only difference is that $N_{s, i}$ and $N_{s, 1}$ in \cref{lem:arm error} are replaced with their lower bounds, based on the following lemma.

\begin{lemma}
\label{lem:shvar pulls} Fix stage $s$ and arm $i \in \cA_s$ in \shvar. Then
\begin{align*}
  N_{s, i}
  \geq \frac{\sigma_i^2}{\sigma_{\max}^2} \left(\frac{n_s}{\abs{\cA_s}} - 1\right)\,,
\end{align*}
where $\sigma_{\max} = \max_{i \in \cA} \sigma_i$ is the maximum reward noise and $n_s$ is the budget in stage $s$.
\end{lemma}
\begin{proof}
Let $J$ be the most pulled arm in stage $s$ and $\ell \in [n_s]$ be the round where arm $J$ is pulled the last time. By the design of \shvar, since arm $J$ is pulled in round $\ell$,
\begin{align*}
  \frac{\sigma_J^2}{N_{s, \ell, J}}
  \geq \frac{\sigma_i^2}{N_{s, \ell, i}}
\end{align*}
holds for any arm $i \in \cA_s$. This can be further rearranged as
\begin{align*}
  N_{s, \ell, i}
  \geq \frac{\sigma_i^2}{\sigma_J^2} N_{s, \ell, J}\,.
\end{align*}
Since arm $J$ is the most pulled arm in stage $s$ and $\ell$ is the round of its last pull,
\begin{align*}
  N_{s, \ell, J}
  = N_{s, J} - 1
  \geq \frac{n_s}{\abs{\cA_s}} - 1\,.
\end{align*}
Moreover, $N_{s, i} \geq N_{s, \ell, i}$. Now we combine all inequalities and get
\begin{align}
  N_{s, i}
  \geq \frac{\sigma_i^2}{\sigma_J^2} \left(\frac{n_s}{\abs{\cA_s}} - 1\right)\,.
  \label{eq:shvar pull lower bound}
\end{align}
To eliminate dependence on random $J$, we use $\sigma_J \leq \sigma_{\max}$. This concludes the proof.
\end{proof}

When plugged into \cref{lem:arm error}, we get
\begin{align*}
  \prob{\hat{\mu}_{s, i} \geq \hat{\mu}_{s, 1}}
  \leq \exp\left[- \frac{\Delta_i^2}
  {2 \left(\frac{\sigma_i^2}{N_{s, i}} + \frac{\sigma_1^2}{N_{s, 1}}\right)}\right]
  \leq \exp\left[- \frac{\left(\frac{n_s}{\abs{\cA_s}} - 1\right) \Delta_i^2}
  {4 \sigma_{\max}^2}\right]\,.
\end{align*}
This completes the proof.

\section{Proof of Theorem \ref{thm:shadavar}}
\label{sec:shadavar proof}

This proof has the same steps as that in \cref{sec:shvar proof}. The main difference is that $N_{s, i}$ and $N_{s, 1}$ in \cref{lem:arm error} are replaced with their lower bounds, based on the following lemma.

\begin{lemma}
\label{lem:shadavar pulls} Fix stage $s$ and arm $i \in \cA_s$ in \shadavar. Then
\begin{align*}
  N_{s, i}
  \geq \frac{\sigma_i^2}{\sigma_{\max}^2} \alpha(\abs{\cA_s}, n_s, \delta)
  \left(\frac{n_s}{\abs{\cA_s}} - 1\right)\,,
\end{align*}
where $\sigma_{\max} = \max_{i \in \cA} \sigma_i$ is the maximum reward noise, $n_s$ is the budget in stage $s$, and
\begin{align*}
  \alpha(k, n, \delta)
  = \frac{1 - 2 \sqrt{\frac{\log(1 / \delta)}{n / k - 2}}}
  {1 + 2 \sqrt{\frac{\log(1 / \delta)}{n / k - 2}} +
  \frac{2 \log(1 / \delta)}{n / k - 2}}
\end{align*}
is an arm-independent constant.
\end{lemma}
\begin{proof}
Let $J$ be the most pulled arm in stage $s$ and $\ell \in [n_s]$ be the round where arm $J$ is pulled the last time. By the design of \shadavar, since arm $J$ is pulled in round $\ell$,
\begin{align*}
  \frac{U_{s, \ell, J}}{N_{s, \ell, J}}
  \geq \frac{U_{s, \ell, i}}{N_{s, \ell, i}}
\end{align*}
holds for any arm $i \in \cA_s$. Analogously to \eqref{eq:shvar pull lower bound}, this inequality can be rearranged and loosened as
\begin{align}
  N_{s, i}
  \geq \frac{U_{s, \ell, i}}{U_{s, \ell, J}} \left(\frac{n_s}{\abs{\cA_s}} - 1\right)\,.
  \label{eq:shadavar pull lower bound}
\end{align}
We bound $U_{s, \ell, i}$ from below using the fact that $U_{s, \ell, i} \geq \sigma_i^2$ holds with probability at least $1 - \delta$, based on the first claim in \cref{lem:concentration}. To bound $U_{s, \ell, J}$, we apply the second claim in \cref{lem:concentration} to bound $\hat{\sigma}_{s, \ell, J}^2$ in $U_{s, \ell, J}$, and get that
\begin{align*}
  U_{s, \ell, J}
  \leq \sigma_J^2 \frac{1 + 2 \sqrt{\frac{\log(1 / \delta)}{N_{s, \ell, J} - 1}} +
  \frac{2 \log(1 / \delta)}{N_{s, \ell, J} - 1}}
  {1 - 2 \sqrt{\frac{\log(1 / \delta)}{N_{s, \ell, J} - 1}}}
\end{align*}
holds with probability at least $1 - \delta$. Finally, we plug both bounds into \eqref{eq:shadavar pull lower bound} and get
\begin{align*}
  N_{s, i}
  \geq \frac{\sigma_i^2}{\sigma_J^2}
  \frac{1 - 2 \sqrt{\frac{\log(1 / \delta)}{N_{s, \ell, J} - 1}}}
  {1 + 2 \sqrt{\frac{\log(1 / \delta)}{N_{s, \ell, J} - 1}} +
  \frac{2 \log(1 / \delta)}{N_{s, \ell, J} - 1}} \left(\frac{n_s}{\abs{\cA_s}} - 1\right)\,.
\end{align*}
To eliminate dependence on random $J$, we use that $\sigma_J \leq \sigma_{\max}$ and $N_{s, \ell, J} \geq n_s / \abs{\cA_s} - 1$. This yields our claim and concludes the proof of \cref{lem:shadavar pulls}.
\end{proof}

Similarly to \cref{lem:shvar pulls}, this bound is asymptotically tight when all reward variances are identical. Also $\alpha(\abs{\cA_s}, n_s, \delta) \to 1$ as $n_s \to \infty$. Therefore, the bound has the same shape as that in \cref{lem:shvar pulls}.

The application of \cref{lem:shadavar pulls} requires more care. Specifically, it relies on high-probability confidence intervals derived in \cref{lem:concentration}, which need $N_{s, t, i} > 4 \log(1 / \delta) + 1$. This is guaranteed whenever $n \geq K \log_2 K (4 \log(1 / \delta) + 1)$. Moreover, since the confidence intervals need to hold in any stage $s$ and round $t$, and for any arm $i$, we need a union bound over $K n$ events. This leads to the following claim.

Suppose that $n \geq K \log_2 K (4 \log(1 / \delta) + 1)$. Then, when \cref{lem:shadavar pulls} is plugged into \cref{lem:arm error}, we get that
\begin{align*}
  \prob{\hat{\mu}_{s, i} \geq \hat{\mu}_{s, 1}}
  \leq \exp\left[- \frac{\Delta_i^2}
  {2 \left(\frac{\sigma_i^2}{N_{s, i}} + \frac{\sigma_1^2}{N_{s, 1}}\right)}\right]
  \leq \exp\left[- \frac{\alpha(\abs{\cA_s}, n_s, K n \delta)
  \left(\frac{n_s}{\abs{\cA_s}} - 1\right) \Delta_i^2}
  {4 \sigma_{\max}^2}\right]\,.
\end{align*}
This completes the proof.

\end{document}